\documentclass[11pt]{article}
\usepackage{mathrsfs,amsmath,amsfonts,amssymb,bm,bbm,eufrak,dsfont,pifont,amscd,
stmaryrd,euscript,amsthm,appendix,color,epsfig,xr}

\usepackage{wrapfig}
\usepackage{subcaption}
\usepackage{graphicx,color,authblk}
\usepackage{amsfonts}
\usepackage{bbm}
\usepackage{amsmath}
\usepackage{tikz}
\usepackage{hyperref}
\newtheorem{theorem}{Theorem}
\newtheorem{lemma}[theorem]{Lemma}
\newtheorem{corollary}[theorem]{Corollary}
\newtheorem{proposition}[theorem]{Proposition}
\theoremstyle{remark}
\newtheorem{remark}{Remark}
\theoremstyle{example}

\theoremstyle{definition}

\theoremstyle{remark}

\definecolor{sylvain}{rgb}{0,.6,0}

\definecolor{orange}{rgb}{1,0.5,0}

\def\bydef{:=}
\def\F{\mathcal{F}}

\def\X{\mathcal{X}}

\def\Z{\mathcal{Z}}
\def\R{\mathcal{R}}

\def\GAN{\mathrm{GAN}}
\def\fGAN{\mathrm{f,GAN}}
\def\KL{\mathrm{KL}}
\def\JS{\mathrm{JS}}
\def\VAE{\mathrm{VAE}}
\def\POT{\mathrm{POT}}
\def\AAE{\mathrm{AAE}}
\def\WGAN{\mathrm{WGAN}}
\def\AVB{\mathrm{AVB}}
\def\supp{\mathbf{supp}\,}

\def\H{\mathcal{H}}

\newcommand{\oo}[1]{\text{\sf 1}_{[#1]}}

\def\E{\mathbb{E}}

\newcommand{\e}[1]{\mathbb{E}\left[#1 \right]}
\newcommand{\ee}[2]{\mathbb{E}_{#1}\left[#2 \right]}

\newcommand\independent{\protect\mathpalette{\protect\independenT}{\perp}}
\def\independenT#1#2{\mathrel{\rlap{$#1#2$}\mkern2mu{#1#2}}}

\title{From optimal transport to generative modeling:\\the VEGAN cookbook}

\author[1]{Olivier Bousquet}
\author[1]{Sylvain Gelly}
\author[2]{Ilya Tolstikhin}
\author[2]{Carl-Johann Simon-Gabriel}
\author[2]{Bernhard Sch\"olkopf}
\affil[1]{Google Brain}
\affil[2]{Max Planck Institute for Intelligent Systems}

\date{}
\begin{document}

\maketitle
\begin{abstract}
We study unsupervised generative modeling in terms of the optimal transport (OT) problem between true (but unknown) data distribution $P_X$ and the \emph{latent variable model} distribution $P_G$. 
We show that the OT problem can be equivalently written in terms of probabilistic encoders, which are constrained to match the posterior and prior distributions over the latent space.
When relaxed, this constrained optimization problem leads to a \emph{penalized optimal transport} (POT) objective, which can be efficiently minimized using stochastic gradient descent by sampling from $P_X$ and $P_G$.
We show that POT for the 2-Wasserstein distance coincides with the objective heuristically employed in adversarial auto-encoders (AAE) \cite{MSJ+16}, which provides the first theoretical justification for AAEs known to the authors.
We also compare POT to other popular techniques like variational auto-encoders (VAE)~\cite{KW14}. Our theoretical results include (a) a better understanding of the commonly observed blurriness of images generated by VAEs, and (b) establishing duality between Wasserstein GAN~\cite{AB17} and POT for the 1-Wasserstein distance.

\end{abstract}
\section{Introduction}
The field of representation learning was initially driven by supervised approaches, with impressive results using large labelled datasets. Unsupervised generative modeling, in contrast, used to be a domain governed by probabilistic approaches focusing on low-dimensional data. 
Recent years have seen a convergence of those two approaches. In the new field that formed at the intersection, variational autoencoders (VAEs) \cite{KW14} form one well-established approach, theoretically elegant yet with the drawback that they tend to generate blurry images. In contrast, generative adversarial networks (GANs) \cite{goodfellow2014generative} turned out to be more impressive in terms of the visual quality of images sampled from the model, but have been reported harder to train. There has been a flurry of activity in assaying numerous configurations of GANs as well as combinations of VAEs and GANs. A unifying theory
relating GANs to VAEs in a principled way is yet to be discovered. This forms a major motivation for the studies underlying the present paper.

Following \cite{AB17}, we approach generative modeling from the optimal transport point of view. 
The optimal transport (OT) cost \cite{V03} is a way to measure a distance between probability distributions and provides a much weaker topology than many others, including $f$-divergences associated with the original GAN algorithms.
This is particularly important in applications, where data is usually supported on low dimensional manifolds in the input space $\X$.
As a result, stronger notions of distances (such as $f$-divergences, which capture the density ratio between distributions) often max out, providing no useful gradients for training.
In contrast, the optimal transport behave nicer and may lead to a more stable training \cite{AB17}.

In this work we aim at minimizing the optimal transport cost $W_c(P_X,P_G)$ between the true (but unknown) data distribution $P_X$ and a \emph{latent variable model} $P_G$.
We do so via the primal form and investigate theoretical properties of this optimization problem. 
Our main contributions are listed below; cf.\ also Figure \ref{fig:main}.
\begin{enumerate}
\item
We derive an equivalent formulation for the \emph{primal form} of $W_c(P_X, P_G)$, which makes the role of latent space $\Z$ and probabilistic encoders $Q(Z|X)$ explicit (Theorem \ref{thm:main}).
\item
Unlike in VAE, we arrive at an optimization problem where the $Q$ are constrained.
We relax the constraints by penalization, arriving at the \emph{penalized optimal transport} (POT) objective \eqref{eq:our-algo} which can be minimized with stochastic gradient descent by sampling from $P_X$~and~$P_G$.
\item
We show that for squared Euclidean cost $c$ (Section \ref{sec:VAEs}), POT coincides with the objective of adversarial auto-encoders (AAE) \cite{MSJ+16}. We believe this provides the first theoretical justification for AAE, showing that they approximately minimize the 2-Wasserstein distance $W_2(P_X,P_G)$.
We also compare POT to VAE, adversarial variational Bayes (AVB)~\cite{MNG17}, and other methods based on the marginal log-likelihood.
In particular, we show that all these methods \emph{necessarily} suffer from blurry outputs, unlike POT or AAE.
\item
When $c$ is the Euclidean distance (Section \ref{sec:WGAN}), POT and WGAN \cite{AB17} both minimize the 1-Wasserstein distance $W_1(P_X,P_G)$.
They approach this problem from primal/dual forms respectively, which leads to a different behaviour of the resulting algorithms.
\end{enumerate}

\begin{figure}[t]
\hspace{0.01\linewidth}
\begin{subfigure}[t]{0.45\linewidth}
    \caption{VAE and AVB}
    \includegraphics[width=1.0\linewidth]{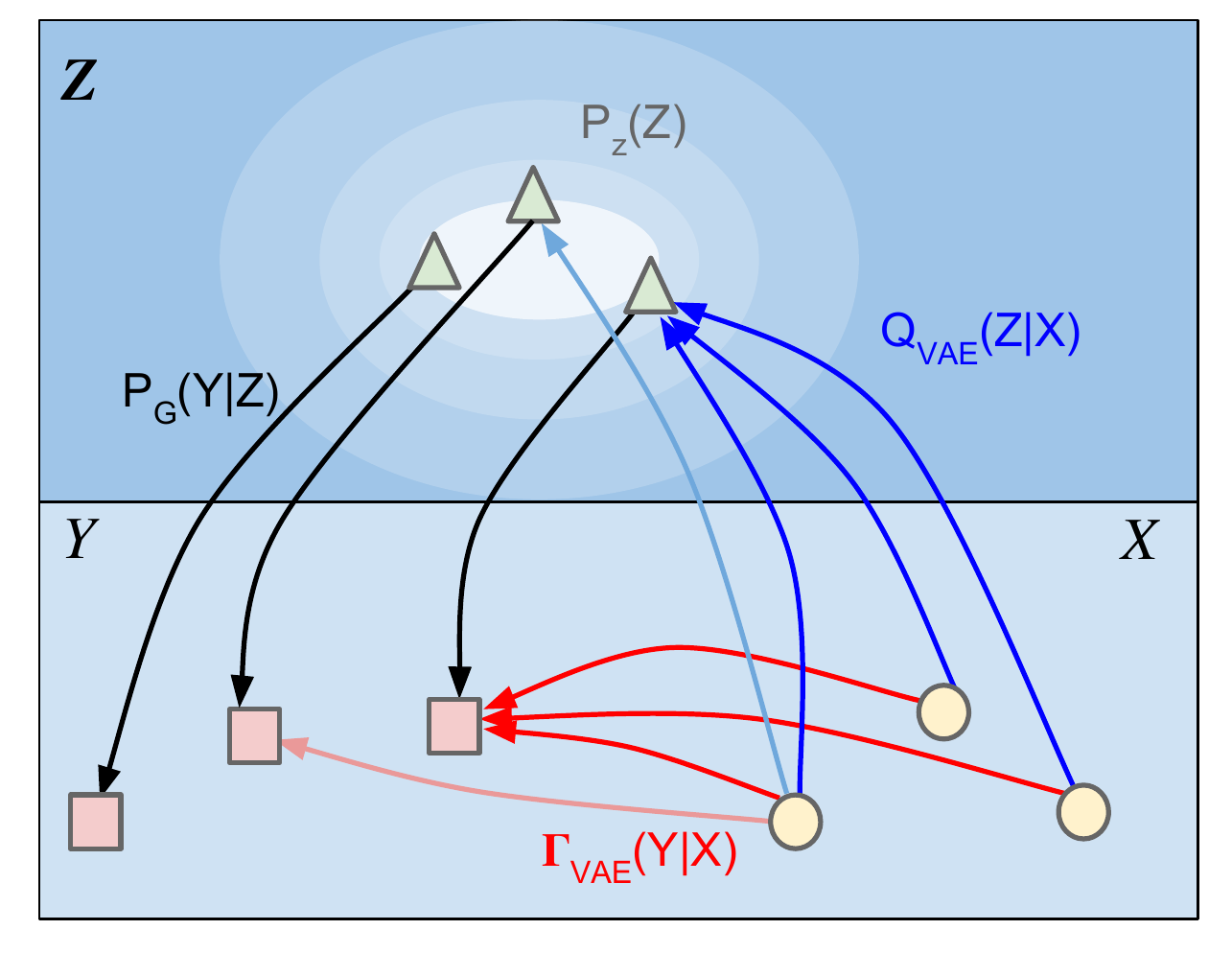}
  \end{subfigure}
  \hspace{0.05\linewidth}
  \begin{subfigure}[t]{0.45\linewidth}
    \caption{Optimal transport (primal form) and AAE}
    \includegraphics[width=1.0\linewidth]{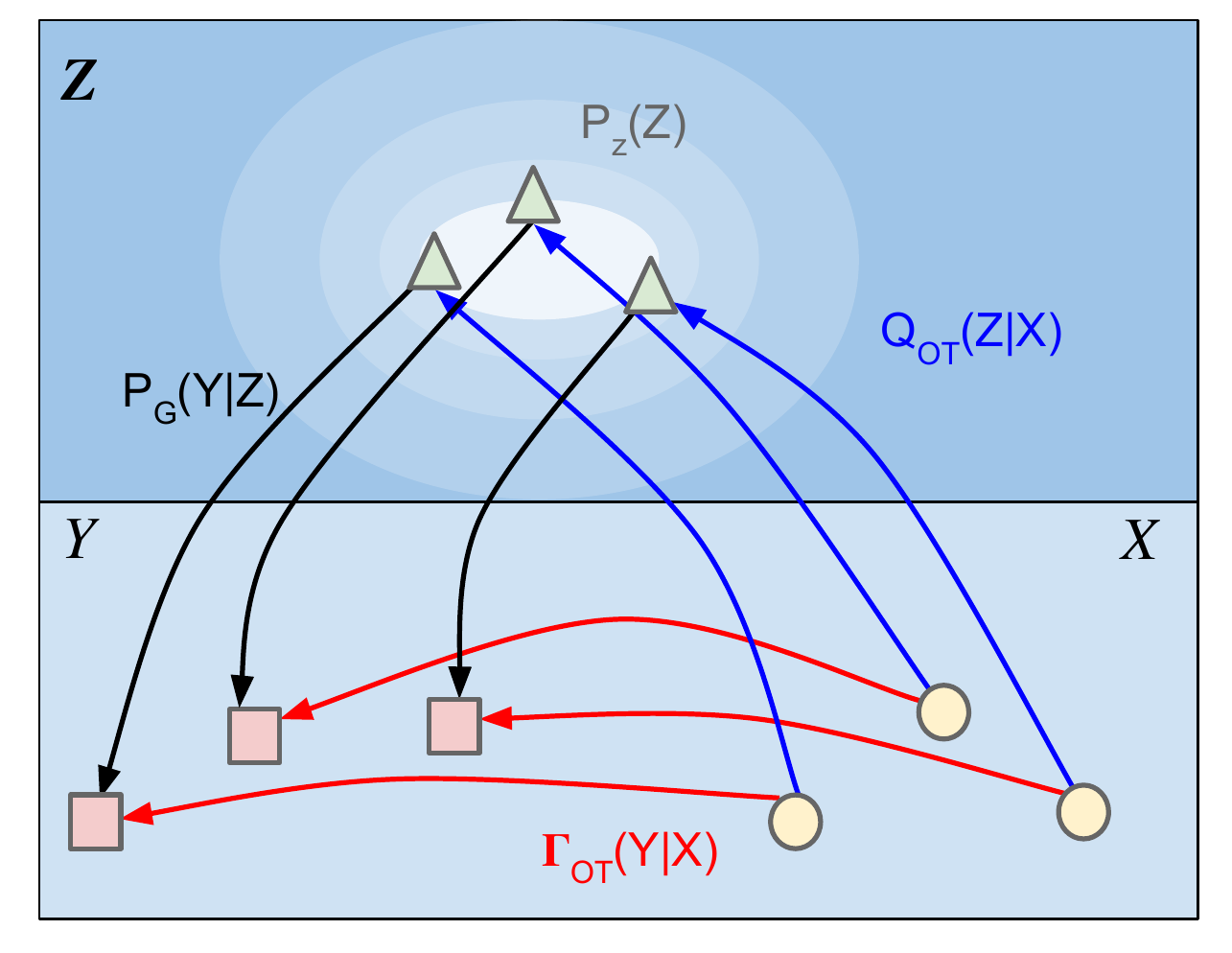}
  \end{subfigure}

  \caption{Different behaviours of generative models. The top half represents the latent space $\Z$ with codes (triangles) sampled from $P_Z$. The bottom half represents the data space $\X$, with true data points $X$ (circles) and generated ones $Y$ (squares). The arrows represent the conditional distributions. Generally these are not one to one mappings, but for improved readability we show only one or two arrows to the most likely points. 
On the \textbf{left figure}, describing VAE \cite{KW14} and AVB \cite{MNG17}, $\Gamma_{\VAE}(Y|Z)$ is a composite of the encoder $Q_{\VAE}(Z|X)$ and the decoder $P_G(Y|Z)$, mapping each true data point $X$ to a distribution on generated points $Y$. For a fixed decoder, the optimal encoder $Q^*_{\VAE}$ will assign mass proportionally to the distance between $Y$ and $X$ and the probability $P_Z(Z)$ (see Eq.\,\ref{eq:vae-opt-q}). 
We see how different points $X$ are mapped with high probability to the same $Y$, while the other generated points $Y$ are reached only with low probabilities.
On the \textbf{right figure}, the OT is expressed as a conditional mapping $\Gamma_{\mathrm{OT}}(Y | X)$. One of our main results (Theorem \ref{thm:main}) shows that this mapping  can be reparametrized via transport $X\to Z \to Y$, making explicit a role of the encoder $Q_{\mathrm{OT}}(Z|X)$.}
  \label{fig:main}
\end{figure}

Finally, following a somewhat well-established tradition of using acronyms based on the ``GAN'' suffix, and because we give a recipe for blending VAE and GAN (using POT), we propose to call this work a VEGAN cookbook.

\paragraph{Related work}
The authors of \cite{GCPB16} address computing the OT cost in large scale using stochastic gradient descent (SGD) and sampling.
They approach this task either through the dual formulation, or via a regularized version of the primal.
They do not discuss any implications for generative modeling.
Our approach is based on the primal form of OT, we arrive at regularizers which are very different, and our main focus is on generative modeling.
The Wasserstein GAN \cite{AB17} minimizes the 1-Wasserstein distance $W_1(P_X,P_G)$ for generative modeling.
The authors approach this task from the dual form.
Unfortunately, their algorithm cannot be readily applied to any other OT cost, because the famous Kantorovich duality holds only for $W_1$.
In contrast, our algorithm POT approaches the same problem from the primal form and can be applied for any cost function $c$.

\medskip
The present paper is structured as follows.
In Section \ref{sec:preliminary} we introduce our notation and discuss existing generative modeling techniques, including GANs, VAEs, and AAE.
Section \ref{sec:main} contains our main results, including a theoretical analysis of the primal form of OT and the novel POT objective.
Section~\ref{sec:implications} discusses the implications of our new results.
Finally, we discuss future work in Section~\ref{sec:conclusion}.
Proofs may be found in Appendix \ref{appendix:proofs}.

\section{Notations and preliminaries}
\label{sec:preliminary}
We use calligraphic letters (i.e.\:$\X$) for sets,
capital letters (i.e.\:$X$) for random variables,
and lower case letters (i.e.\:$x$) for their values.
We denote probability distributions with capital letters (i.e.\:$P(X)$) and densities 
with lower case letters (i.e.\:$p(x)$).
By $\delta_x$ we denote the Dirac distribution putting mass 1 on $x\in\X$, and $\supp P$ denotes the support of $P$.

We will often need to measure the agreement between two probability distributions $P$ and $Q$ and there are many ways to do so. 
The class of \emph{$f$-divergences}~\cite{LM08} is defined by $D_f(P\|Q):=\int f\bigl(\frac{p(x)}{q(x)}\bigr)q(x)dx$, where $f\colon (0,\infty)\to \R$ is any convex function satisfying $f(1)=0$. 
It is known that $D_f\geq 0$ and $D_f=0$ if $P=Q$.
Classical examples include the Kullback-Leibler $D_\KL$ and Jensen-Shannon $D_\JS$ divergences.
Another rich class of divergences is induced by the \emph{optimal transport} (OT) problem \cite{V03}.
Kantorovich's formulation~\cite{K42} of the problem is given by
\begin{equation}
\label{eq:ot}
W_c(P,Q):=\inf_{\Gamma\in \mathcal{P}(X\sim P,Y\sim Q)} \E_{(X,Y)\sim\Gamma}[c(X,Y)]\,,
\end{equation}
where $c(x,y)\colon \X \times \X \to \R_+$ is any measurable cost function
and $\mathcal{P}(X\sim P,Y\sim Q)$ is a set of all joint distributions of $(X,Y)$ with marginals $P$ and $Q$ respectively. 
A particularly interesting case is when $(\X,d)$ is a metric space and $c(x,y) = d^p(x,y)$ for $p\geq 1$.
In this case $W_{p}$, the $p$-th root of~$W_c$, is called \emph{the $p$-Wasserstein distance}.
Finally, the Kantorovich-Rubinstein theorem establishes a duality for the 1-Wasserstein distance, which holds under mild assumptions on $P$ and $Q$:
\begin{equation}
\label{eq:KRD}
W_{1}(P,Q) = \sup_{f\in \mathcal{F}_L} \bigl| \E_{X\sim P}[f(X)] - \E_{Y \sim Q}[ f(Y)] \bigr|,
\end{equation}
where $\mathcal{F}_L$ is the class of all bounded 1-Lipschitz functions on $(\X,d)$. 
Note that the same symbol is used for $W_p$ and $W_c$, but only $p$ is a number and thus the above $W_1$ refers to the Wasserstein distance.



\subsection{Implicit generative models: a short tour of GANs and VAEs}
\label{sec:vae+gan}
Even though GANs and VAEs are quite different---both in terms of the conceptual frameworks and empirical performance---they share important features: (a) both can be trained by sampling from the model $P_G$ without knowing an analytical form of its density and (b) both can be scaled up with SGD.
As a result, it becomes possible to use highly flexible \emph{implicit} models $P_G$ defined by a two-step procedure, where first a code $Z$ is sampled from a fixed distribution $P_Z$ on a latent space $\Z$ and then $Z$ is mapped to the image $G(Z)\in\X=\R^d$ with a (possibly random) transformation $G\colon\Z\to\X$.
This results in \emph{latent variable models} $P_G$ defined on $\X$ with density of the form
\begin{equation}
\label{eq:latent-var}
p_G(x):= \int_{\Z} p_G(x|z)p_z(z) d z,\quad \forall x\in\X,
\end{equation}
assuming all involved densities are properly defined.
These models are indeed easy to sample and, provided $G$ can be differentiated analytically with~respect to its parameters, $P_G$ can be trained with SGD.
The field is growing rapidly and numerous variations of VAEs and GANs are available in the literature. 
Next we introduce and compare several of them.

The original {\bf generative adversarial network} (GAN) \cite{goodfellow2014generative} approach minimizes
\begin{equation}
\label{eq:GAN}
D_\GAN(P_X, P_G) = 
\sup_{T\in\mathcal{T}}
\E_{X\sim P_X}[\log T(X)] + \E_{Z\sim P_Z}\bigl[\log\bigl(1 - T(G(Z))\bigr)\bigr]
\end{equation}
with respect to a deterministic \emph{generator} $G\colon\Z\to\X$, where $\mathcal{T}$ is any non-parametric class of choice.
It is known that $D_\GAN(P_X, P_G)\leq2\cdot D_\JS(P_X, P_G) - \log(4)$ and the inequality turns into identity in the \emph{nonparametric limit}, that is when the class $\mathcal{T}$ becomes rich enough to represent \emph{all} functions mapping $\X$ to $(0,1)$.
Hence, GANs are \emph{minimizing a lower bound} on the JS-divergence.
However, GANs are not only linked to the JS-divergence: the $f$-GAN approach \cite{nowozin2016f} showed that a slight modification $D_\fGAN$ of the objective \eqref{eq:GAN} allows to lower bound any desired $f$-divergence in a similar way.
In practice, both generator $G$ and \emph{discriminator} $T$ are trained in alternating SGD steps.
Stopping criteria as well as adequate evaluation of the trained GAN models remain open questions.

Recently, the authors of \cite{AB17} argued that the 1-Wasserstein
distance $W_{1}$, which is known to induce a much weaker topology than
$D_\JS$, may be better suited for generative modeling.  When $P_X$ and
$P_G$ are supported on largely disjoint low-dimensional manifolds
(which may be the case in applications), $D_\KL$, $D_\JS$, and other
strong distances between $P_X$ and $P_G$ max out and no longer provide
useful gradients for $P_G$.  This ``vanishing gradient'' problem
necessitates complicated scheduling between the $G$/$T$ updates.
In~contrast, $W_{1}$ is still sensible in these cases and provides
stable gradients.  The {\bf Wasserstein GAN} (WGAN) minimizes
\[
D_\WGAN(P_X, P_G) = 
\sup_{T\in\mathcal{W}}
\E_{X\sim P_X}[T(X)] - \E_{Z \sim P_Z}\bigl[ T(G(Z))\bigr],
\]
where $\mathcal{W}$ is any subset of 1-Lipschitz functions on $\X$.
It follows from \eqref{eq:KRD} that $D_\WGAN(P_X, P_G) \leq W_{1}(P_X, P_G)$ and thus WGAN is \emph{minimizing a lower bound} on the 1-Wasserstein distance.

{\bf Variational auto-encoders} (VAE) \cite{KW14} utilize models $P_G$ of the form \eqref{eq:latent-var}
and minimize
\begin{equation}
\label{eq:VAE}
D_\VAE(P_X,P_G) = 
\inf_{Q(Z|X)\in\mathcal{Q}} 
\E_{P_X}\left[
D_\KL\bigl(Q(Z|X), P_Z\bigr)
-
\E_{Q(Z|X)}[\log p_G(X|Z)]
\,
\right]
\end{equation}
with respect to a random \emph{decoder} mapping $P_G(X|Z)$.
The conditional distribution $P_G(X|Z)$ is often parametrized by a deep net $G$ and can have any form as long as its density $p_G(x|z)$ can be computed and differentiated with respect to the parameters of $G$.
A typical choice is to use Gaussians $P_G(X|Z) = \mathcal{N}(X;G(Z), \sigma^2\cdot I)$.
If $\mathcal{Q}$ is the set of \emph{all} conditional probability distributions $Q(Z|X)$, 
the objective of VAE coincides with the negative marginal log-likelihood $D_\VAE(P_X, P_G) = -\E_{P_X}[\log P_G(X)]$.
However, in order to make the $D_\KL$ term of~\eqref{eq:VAE} tractable in closed form,
the original implementation of VAE uses a standard normal $P_Z$ and restricts $\mathcal{Q}$ to a class of Gaussian distributions $Q(Z|X) = \mathcal{N}\bigl(Z; \mu(X), \Sigma(X)\bigr)$ with mean $\mu$ and diagonal covariance~$\Sigma$ parametrized by deep nets.
As a consequence, VAE is \emph{minimizing an upper bound} on the negative log-likelihood or, equivalently, on the KL-divergence $D_\KL(P_X, P_G)$.
Further details can be found in Section \ref{appendix:VAE}.

One possible way to reduce the gap between the true negative log-likelihood and the upper bound provided by $D_\VAE$ is to enlarge the class $\mathcal{Q}$.
{\bf Adversarial variational Bayes} (AVB)~\cite{MNG17} follows this argument by employing the idea of GANs.
Given any point $x\in\X$, a noise $\epsilon\sim\mathcal{N}(0,1)$, and any fixed transformation $e\colon \X\times\R\to\Z$, a random variable $e(x, \epsilon)$ implicitly defines one particular conditional distribution $Q_e(Z|X=x)$.
AVB allows $\mathcal{Q}$ to contain all such distributions for different choices of $e$,
replaces the intractable term $D_\KL\bigl(Q_e(Z|X), P_Z\bigr)$ in \eqref{eq:VAE} by the adversarial approximation $D_\fGAN$ corresponding to the KL-divergence, and proposes to minimize\footnote{
The authors of AVB \cite{MNG17} note that using $f$-GAN as described above actually results in ``unstable training''.
Instead, following the approach of \cite{PAS+16}, they use a trained discriminator $T^*$ resulting from the $D_\GAN$ objective~\eqref{eq:GAN} to approximate the ratio of densities and then directly estimate the KL divergence $\int f\bigl(p(x)/q(x)\bigr)q(x)dx$.}
\begin{equation}
\label{eq:obj-AVB}
D_\AVB(P_X, P_G) = 
\inf_{Q_e(Z|X)\in\mathcal{Q}} 
\E_{P_X}\left[
D_\fGAN\bigl(Q_e(Z|X), P_Z\bigr)
-
\E_{Q_e(Z|X)}[\log p_G(X|Z)]\,
\right].
\end{equation}

The $D_\KL$ term in \eqref{eq:VAE} may be viewed as a regularizer. 
Indeed, VAE reduces to the classical unregularized auto-encoder if this term is dropped, minimizing the reconstruction cost of the encoder-decoder pair $Q(Z|X), P_G(X|Z)$. 
This often results in different training points being encoded into non-overlapping zones chaotically scattered all across the $\Z$ space with ``holes'' in between where the decoder mapping $P_G(X|Z)$ has never been trained.
Overall, the encoder $Q(Z|X)$ trained in this way does not provide a useful representation and sampling from the latent space $\Z$ becomes hard \cite{BCV13}.

{\bf Adversarial auto-encoders} (AAE) \cite{MSJ+16}  
replace the $D_\KL$ term in~\eqref{eq:VAE} with another regularizer:
\begin{equation}
\label{eq:obj-AAE}
D_\AAE(P_X, P_G) = 
\inf_{Q(Z|X)\in\mathcal{Q}} 
D_\GAN(Q_Z, P_Z)
-
\E_{P_X}\E_{Q(Z|X)}[\log p_G(X|Z)],
\end{equation}
where $Q_Z$ is the marginal distribution of $Z$ when first $X$ is sampled from $P_X$ and then $Z$ is sampled from $Q(Z|X)$, also known as the \emph{aggregated posterior} \cite{MSJ+16}.
Similarly to AVB, there is no clear link to log-likelihood, as $D_\AAE \leq D_{\AVB}$ (see Appendix \ref{appendix:VAE}).
The~authors of~\cite{MSJ+16} argue that matching  $Q_Z$ to $P_Z$ in this way ensures that there are no ``holes'' left in the latent space $\Z$ and $P_G(X|Z)$ generates reasonable samples whenever $Z\sim P_Z$.
They also report an equally good performance of different types of conditional distributions $Q(Z|X)$, including Gaussians as used in VAEs, implicit models $Q_e$ as used in AVB, and \emph{deterministic} encoder mappings, i.e.\:$Q(Z| X) = \delta_{\mu(X)}$ with $\mu\colon\X\to\Z$.
 
\section{Minimizing the primal of optimal transport}
\label{sec:main}
We have argued that minimizing the optimal transport cost $W_c(P_X, P_G)$ between the true data distribution $P_X$ and the model~$P_G$ is a reasonable goal for generative modeling. We now will explain how this can be done in the primal
formulation of the OT problem \eqref{eq:ot} by
reparametrizing the space of couplings (Section \ref{sec:reparam}) and relaxing the marginal constraint (Section \ref{sec:consequences}), leading to a formulation involving expectations over $P_X$ and $P_G$ that can thus be solved using SGD and sampling.

\subsection{Reparametrization of the couplings}
\label{sec:reparam}
We will consider certain sets of joint probability distributions of three random variables $(X,Y,Z)\in\X\times\X\times\Z$.
The reader may wish to think of $X$ as true images, $Y$ as images sampled from the model, and $Z$ as latent codes.
We denote by $P_{G,Z}(Y,Z)$ a joint distribution of a variable pair $(Y,Z)$, where $Z$ is first sampled from $P_Z$ and next $Y$ from $P_G(Y|Z)$.
Note that $P_G$ defined in \eqref{eq:latent-var} and used throughout this work is the marginal distribution of $Y$ when $(Y,Z)\sim P_{G,Z}$.

In the optimal transport problem, we consider joint distributions $\Gamma(X,Y)$ which are
called~\emph{couplings} between values of $X$ and~$Y$. Because of the marginal constraint, we can
write $\Gamma(X,Y)=\Gamma(Y|X)P_X(X)$ and we can consider $\Gamma(Y|X)$ as a non-deterministic mapping from $X$ to $Y$.
In this section we will show how to~\emph{factor} this mapping through $\Z$, i.e., decompose it
into an encoding distribution $Q(Z|X)$ and the generating distribution $P_G(Y|Z)$.

In order to give a more intuitive explanation of this decomposition, consider the case where all probability distributions have densities with respect to the Lebesgue measure.
In this case our results show that \emph{some} elements $\Gamma$ of $\mathcal{P}(X\sim P_X, Y\sim P_G)$ have densities of the form
\begin{equation}
\label{eq:densities}
\gamma(x,y) = \int_{\Z} p_G(y|z)q(z|x)p_X(x)dz,
\end{equation}
where the density of the conditional distribution $Q(Z|X)$ satisfies $q_Z(z):=\int_{\X} q(z|x)p_X(x)dx = p_Z(z)$ for all $z\in\Z$.
Equation \eqref{eq:densities} allows to express the search space (couplings) of the OT problem in terms of the probabilistic encoders $Q(Z|X)$.
Unlike VAE, AVB, and other methods based on the marginal log-likelihood, where the $Q$ are not constrained, the encoders of the OT problem need to match the aggregated posterior $Q_Z$ to the prior $P_Z$.

\paragraph{Formal statement}
As in Section~\ref{sec:preliminary}, $\mathcal{P}(X\sim P_X, Y\sim P_G)$ denotes the set of all joint distributions of $(X,Y)$ with marginals $P_X,P_G$, and likewise for $\mathcal{P}(X\sim P_X, Z\sim P_Z)$.
The set of all joint distributions of $(X,Y,Z)$ such that $X\sim P_X$, $(Y,Z)\sim P_{G,Z}$, and $(Y\independent X)|Z$ will be denoted by~$\mathcal{P}_{X,Y,Z}$.
Finally, we denote by $\mathcal{P}_{X,Y}$ and $\mathcal{P}_{X,Z}$ the sets of marginals on $(X,Y)$ and $(X,Z)$ (respectively) induced by distributions in $\mathcal{P}_{X,Y,Z}$. Note that $\mathcal{P}(P_X,P_G)$, $\mathcal{P}_{X,Y,Z}$, and $\mathcal{P}_{X,Y}$ depend on the choice of conditional distributions $P_G(Y|Z)$, while $\mathcal{P}_{X,Z}$ does not.
In fact, it is easy to check that $\mathcal{P}_{X,Z} = \mathcal{P}(X\sim P_X, Z\sim P_Z)$.
From the definitions it is clear that $\mathcal{P}_{X,Y} \subseteq\mathcal{P}(P_X, P_G)$ and we get the following upper bound:
\begin{equation}
\label{eq:thm-main-1}
W_c(P_X, P_G) 
\leq
W_c^\dagger(P_X, P_G)
\bydef \inf_{P\in \mathcal{P}_{X,Y}} \ee{(X, Y)\sim P}{c(X,Y)}
\end{equation}
If $P_G(Y|Z)$~are~Dirac measures (i.e., $Y=G(Z)$), the two sets are actually coincide, thus justifying the reparametrization \eqref{eq:densities} and the illustration in Figure \ref{fig:main}(b), as demonstrated in the following theorem:
\begin{theorem}
\label{thm:main}
If $P_G(Y|Z=z)=\delta_{G(z)}$ for all $z\in \Z$, where $G\colon \Z\to\X$, we have
\begin{align}
\label{eq:thm-main-2}
W_c(P_X, P_G) 
=
W_c^\dagger(P_X, P_G)
&=
\inf_{P\in \mathcal{P}(X\sim P_X, Z\sim P_Z)} \E_{(X,Z)\sim P} \bigl[c\bigl(X,G(Z)\bigr)\bigr]\\
\label{eq:final-constrained-objective}
&= 
\inf_{Q\colon Q_Z =P_Z} \E_{P_X} \E_{Q(Z|X)} \bigl[c\bigl(X,G(Z)\bigr)\bigr],
\end{align}
where $Q_Z$ is the marginal distribution of $Z$ when $X\sim P_X$ and $Z\sim Q(Z|X)$.

\end{theorem}
The r.h.s.\:of \eqref{eq:thm-main-2} is the optimal transport $W_{c_G}(P_X,P_Z)$ between $P_X$ and $P_Z$ with the cost function $c_g(x,z):= c\bigl(x, G(z)\bigr)$ defined on $\X\times \Z$.
If $P_G(Y|Z)$ corresponds to a deterministic mapping $G\colon\Z\to\X$, the~resulting model $P_G$ is the \emph{push-forward of $P_Z$ through $G$}.
Theorem \ref{thm:main} states that in this case the two OT problems are equivalent, $W_c(P_X, P_G) = W_{c_G}(P_X,P_Z)$.

The conditional distributions $Q$ in \eqref{eq:final-constrained-objective} are constrained to ensure that when $X$ is sampled from $P_X$ and then $Z$ from $Q(Z|X)$, the resulting marginal distribution of $Z$ (which was denoted~$Q_Z$ and called \emph{aggregated posterior} in Section \ref{sec:vae+gan}) coincides with $P_Z$.
A very similar result also holds for the case when $P_G(Y|Z)$ are not necessarily Dirac.
Nevertheless, for simplicity we will focus on the Dirac case for now and shortly summarize the more general case in the end of this section.


\subsection{Relaxing the constraints}
\label{sec:consequences}

Minimizing $W_c(P_X, P_G)$ boils down to a min-min optimization problem.
Unfortunately, the constraint on $Q$ makes the variational problem even harder.
We thus propose to replace the constrained optimization problem \eqref{eq:final-constrained-objective} with its relaxed unconstrained version in a standard way.
Namely, use any convex \emph{penalty} $F\colon Q\to\R_+$, such that $F(Q)=0$ if and only if $P_Z=Q_Z$, and for any $\lambda > 0$, consider the following relaxed unconstrained version of $W_{c}^{\dagger}(P_X,P_G)$:
\begin{equation}
\label{eq:final-relaxed-objective}
W_{c}^{\lambda}(P_X,P_G):
=\inf_{Q(Z|X)} \E_{P_X} \E_{Q(Z|X)} \bigl[c\bigl(X,G(Z)\bigr)\bigr] + \lambda F(Q)
\end{equation}
It is well known \cite{BL06} that under mild conditions adding a penalty as in \eqref{eq:final-relaxed-objective} is equivalent to adding a constraint of the form $F(Q) \leq \mu_{\lambda}$ for some $\mu_{\lambda} > 0$.
As $\lambda$ increases, the corresponding $\mu_\lambda$ decreases, and as $\lambda \to \infty$, the solutions of \eqref{eq:final-relaxed-objective} reach the feasible region where $P_Z=Q_Z$.
This shows that $W_{c}^{\lambda}(P_X,P_G) \leq W_{c}(P_X,P_G)$ for all $\lambda \geq 0$ and the gap reduces with increasing $\lambda$.

One possible choice for $F$ is a convex divergence between the prior $P_Z$ and the aggregated posterior~$Q_Z$, such as $D_\JS(Q_Z, P_Z)$, $D_\KL(Q_Z, P_Z)$, or any other member of the $f$-divergence family.
However, this results in intractable $F$. 
Instead, similarly to AVB, we may utilize the adversarial approximation $D_\GAN(Q_Z, P_Z)$, which becomes tight in the nonparametric limit.
We thus arrive at the problem of minimizing a \emph{penalized optimal transport} (POT) objective
\begin{equation}
\label{eq:our-algo}
D_\POT(P_X,P_G):=
\inf_{Q(Z|X)\in\mathcal{Q}} \E_{P_X} \E_{Q(Z|X)} \bigl[c\bigl(X,G(Z)\bigr)\bigr] + \lambda \cdot D_\GAN(Q_Z, P_Z),
\end{equation}
where $\mathcal{Q}$ is any nonparametric set of conditional distributions.
If the cost function $c$ is differentiable, this problem can be solved with SGD similarly to AAE, where we iterate between updating (a) an encoder-decoder pair $Q,G$ and (b) an adversarial discriminator of $D_\GAN$, trying to separate latent codes sampled from $P_Z$ and $Q_Z$.
Moreover, in Section \ref{sec:VAEs} we will show that $D_\POT$ \emph{coincides} with $D_\AAE$ when $c$ is the squared Euclidean cost and the $P_G(Y|Z)$ are Gaussian.
The good empirical performance of AAEs reported in \cite{MSJ+16} provides what we believe to be rather strong support of our theoretical results.

\paragraph{Random decoders}
The above discussion assumed Dirac measures $P_G(Y|Z)$.
If this is not the case, we can still \emph{upper bound} the 2-Wasserstein distance $W_{2}(P_X,P_G)$, corresponding to $c(x,y)=\|x-y\|^2$, in a very similar way to Theorem \ref{thm:main}.
The case of Gaussian decoders $P_G(Y|Z)$, which will be particularly useful when discussing the relation to VAEs,
is summarized in the following remark:
\begin{remark}
\label{remark:gauss}
For $\X = \R^d$ and Gaussian $P_G(Y|Z)=\mathcal{N}(Y;G(Z), \sigma^2\!\cdot\! I_d)$ the value of $W_c(P_X, P_G)$ is upper bounded by $W_c^\dagger(P_X, P_G)$, which coincides with the r.h.s.\:of \eqref{eq:final-constrained-objective} up to a $d\cdot\sigma^2$ additive term (see Corollary \ref{corr:gauss} in Section \ref{appendix:corr-gauss}).
In other words, objective \eqref{eq:final-relaxed-objective}  coincides with the relaxed version of $W_c^\dagger(P_X,P_G)$ up to additive constant,
while $D_\POT$ corresponds to its adversarial approximation.
\end{remark}

\section{Implications: relations to AAE, VAEs, and GANs}
\label{sec:implications}
Thus far we showed that the primal form of the OT problem can be relaxed to allow for efficient minimization by SGD.
We now discuss the main implications of these results.

In~the following Section \ref{sec:VAEs}, we compare minimizing the optimal transport cost $W_c$, the upper bound $W_c^\dagger$, and its relaxed version $D_\POT$ to
VAE, AVB, and AAE in the special case when $c(x,y)=\|x-y\|^2$ and $P_G(Y|Z)=\mathcal{N}(Y; G(Z), \sigma^2\!\cdot\! I)$.
We show that in this case
{\bf (i)} the solutions of VAE and AVB both depend on $\sigma^2$, while the minimizer $G^\dagger$ of $W_c^\dagger(P_X,P_G)$ does not depend on $\sigma^2$, and is the same as the minimizer of $W_c(P_X,P_G)$ for $\sigma^2=0$;
{\bf (ii)} AAE is equivalent to minimizing $D_\POT$ of \eqref{eq:our-algo} with 
$\lambda = 2\sigma^2$.
We also briefly discuss the role of these conclusions in explaining the well-known blurriness of VAE outputs.
Section \ref{sec:WGAN} shows that when $c(x,y)=\|x-y\|$, our algorithm and WGAN approach primal and dual forms respectively of the \emph{same optimization problem}.
Finally, we discuss a difference in behaviour between the two algorithms caused by this duality.

\subsection{The 2-Wasserstein distance: relation to VAE, AVB, and AAE}
\label{sec:VAEs}
Consider the squared Euclidean cost function $c(x,y) = \|x - y\|^2$, for which $W_c$ is the squared 2-Wasserstein distance $W_{2}^2$.
The goal of this section is to compare the minimization of $W_2(P_X, P_G)$ to other generative modeling approaches.
Let us focus our attention on generative distributions $P_G(Y|Z)$ typically used in VAE, AVB, and AAE, i.e., Gaussians $P_G(Y|Z)=\mathcal{N}\bigl(Y; G(Z), \sigma^2\!\cdot\!I_d\bigr)$.
In~order to verify the differentiability of $\log p_G(x|z)$ all three methods require $\sigma^2 > 0$ and have problems handling the case of deterministic decoders ($\sigma^2=0$).
To emphasize the role of the variance $\sigma^2$ we will denote the resulting latent variable model $P_G^{\sigma}$.

\paragraph{Relation to VAE and AVB}
The analysis of Section \ref{sec:main} shows that the value of $W_2(P_X,P_G^\sigma)$ is upper bounded by $W_c^\dagger(P_X,P_G^\sigma)$ of the form \eqref{eq:upper-bound-gauss} and the two coincide when $\sigma^2 = 0$.
Next we summarize properties of solutions $G$ minimizing these two values $W_2$ and $W_c^\dagger$:
\begin{proposition}
\label{prop:vaes}
Let $\X = \R^d$ and assume $c(x,y) = \|x - y\|^2$, $P_G(Y|Z)=\mathcal{N}\bigl(Y; G(Z), \sigma^2\!\cdot\!I\bigr)$ with any function $G\colon\X\to\R$.
If $\sigma^2 > 0$ then the functions $G^*_{\sigma}$ and $G^\dagger$ minimizing $W_c(P_X,P_G^\sigma)$ and $W_c^\dagger(P_X,P_G^\sigma)$ respectively are different: $G^*_{\sigma}$ depends on $\sigma^2$, while $G^\dagger$ does not.
The function $G^\dagger$ is also a minimizer of $W_c(P_X,P_G^0)$.
\end{proposition}

For the purpose of generative modeling, the noise $\sigma^2>0$ is often not desirable, and it is common practice to sample from the trained model $G^*$ by simply returning $G^*(Z)$ for $Z\sim P_Z$ without adding noise to the output. This leads to a mismatch between inference and training. Furthermore, VAE, AVB, and other similar variational methods
implicitly use $\sigma^2$ as a factor to balance the $\ell_2$ reconstruction cost and the KL-regularizer. 

In contrast, Proposition \ref{prop:vaes} shows that for \emph{the same Gaussian models} with \emph{any} given $\sigma^2\geq 0$ we can minimize $W_c^\dagger(P_X, P_G^\sigma)$ and the solution $G^\dagger$ will be indeed the one resulting in the smallest 2-Wasserstein distance between $P_X$ and the noiseless implicit model $G(Z)$, $Z\sim P_Z$ used in practice.

\paragraph{Relation to AAE}
Next, we wish to convey the following intriguing finding. 
Substituting an analytical form of $\log p_G(x|z)$ in \eqref{eq:obj-AAE}, we immediately see that the $D_\AAE$ objective \emph{coincides} with $D_\POT$ up to additive terms independent of $Q$ and $G$ when the regularization coefficient $\lambda$ is~set~to~$2\sigma^2$.

For $0<\sigma^2<\infty$ this means (see Remark \ref{remark:gauss}) that {\em AAE is minimizing the penalized relaxation $D_\POT$ of the constrained optimization problem corresponding to $W_c^\dagger(P_X, P_G^\sigma)$}.
The size of the gap between $D_\POT$ and $W_c^\dagger$ depends on the choice of $\lambda$, i.e., on $\sigma^2$.
If $\sigma^2\to0$, we know (Remark \ref{remark:gauss}) that the upper bound $W_c^\dagger$ converges to the OT cost $W_c$, however the relaxation $D_\POT$ gets loose, as $\lambda=2\sigma^2\to 0$.
In this case AAE approaches the classical unregularized auto-encoder and does not have any connections to the OT problem. 
If $\sigma^2\to\infty$, the solution of the penalized objective $D_\POT$ reaches the feasible region of the original constrained optimization problem \eqref{eq:final-constrained-objective}, because $\lambda=2\sigma^2\to\infty$, and as a result $D_\POT$ converges to $W_c^\dagger(P_X,P_G^\sigma)$.
In this case AAE is searching for the solution $G^\dagger$ of $\min_GW_c^\dagger(P_X,P_G^\sigma)$, 
which is also the function minimizing $W_c(P_X, P_G^0)$ for the deterministic encoder $Y = G(Z)$ according to Proposition \ref{prop:vaes}.
In other words, the function $G^\dagger$ learned by AAE with $\sigma^2\to \infty$ minimizes the 2-Wasserstein distance between $P_X$ and $G(Z)$ when $Z\sim P_Z$.

The authors of \cite{MNG17} tried to establish a connection between AAE and log-likelihood maximization.
They argued that AAE is ``a crude approximation'' to AVB.
Our results suggest that AAE is in fact attempting to minimize the 2-Wasserstein distance between $P_X$ and $P_G^\sigma$, which may explain its good empirical performance reported in~\cite{MSJ+16}.

\paragraph{Blurriness of VAE and AVB}
We next add to the discussion regarding the blurriness commonly attributed to VAE samples.
Our argument shows that VAE, AVB, and other methods based on the marginal log-likelihood \emph{necessarily} lead to an averaging in the input space if $P_G(Y|Z)$ are Gaussian.

First we notice that in the VAE and AVB objectives, for any fixed encoder $Q(Z|X)$, the decoder is minimizing the expected $\ell_2$-reconstruction cost $\E_{P_X}\E_{Q(Z|X)}\bigl[ \|X - G(Z)\|^2\bigr]$ with respect to~$G$.
The optimal solution $G^*$ is of the form $G^*(z) = \E_{P^*_z}[X]$, where $P_z^*(X)\propto P_X(X) Q(Z=z|X)$.
Hence, as soon as $\supp P^*_z$ is non-singleton, the optimal decoder $G^*$ will end up averaging points in the input space.
In particular this will happen whenever there are two points $x_1, x_2$ in $\supp P_X$ such that $\supp Q(Z|X=x_1)$ and $\supp Q(Z|X=x_2)$ overlap.

This overlap necessarily happens in VAEs, which use Gaussian encoders $Q(Z|X)$ supported on the entire $\Z$. 
When probabilistic encoders $Q$ are allowed to be flexible enough, as in AVB, for any fixed $P_G(Y|Z)$ the optimal $Q^*$ will \emph{try to invert the decoder} (see Appendix \ref{appendix:VAE}) and take the form
\begin{equation}
\label{eq:vae-opt-q}
Q^*(Z | X) \approx P_G(Z | X) \bydef \frac{P_G(X | Z) P_Z(Z)}{P_G(X)}.
\end{equation}
This approximation becomes exact in the nonparametric limit of $Q$.
When $P_G(Y|Z)$ is Gaussian we have $p_G(y|z)>0$ for all $y\in \X$ and $z\in \Z$,
showing that $\supp Q^*(Z | X=x)=\supp P_Z$ for all $x\in\X$.
This will again lead to the overlap of encoders if $\supp P_Z=\Z$.
In contrast, the optimal encoders of AAE and POT do not necessarily overlap, as they are not inverting the decoders.

The common belief today is that the blurriness of VAEs is caused by the $\ell_2$ reconstruction cost, or equivalently by the Gaussian form of decoders $P_G(Y|Z)$.
We argue that it is instead caused by the \emph{combination} of (a) Gaussian decoders and (b) the objective (KL-divergence) being minimized.

\subsection{The 1-Wasserstein distance: relation to WGAN}
\label{sec:WGAN}
We have shown that the $D_\POT$ criterion leads to a generalized version of the AAE algorithm and can be seen as a relaxation of the optimal transport cost $W_c$.
In particular, if we choose $c$ to be the Euclidean distance $c(x,y)=\|x-y\|$, we get
a primal formulation of $W_1$. This is the same criterion that WGAN aims to minimize in the dual formulation (see Eq.\,\ref{eq:KRD}).
As a result of Theorem \ref{thm:main}, we have
\[
W_1(P_X,P_G)=\!\! \inf_{Q:Q_Z=P_Z} \!\!\ee{X\sim P_X, Z\sim Q(Z|X)}{\|X-G(Z)\|} = \!\!\sup_{f\in \mathcal{F}_L} \ee{P_X}{f(X)} - \ee{P_Z}{f(G(Z))}.
\]
This means we can now approach the problem of optimizing $W_1$ in two distinct ways, taking gradient steps either in the primal or in the dual forms.
Denote by $Q^*$ the optimal encoder in the primal and $f^*$ the optimal witness function in the dual.
By the envelope theorem, gradients of $W_1$ with respect to $G$ can be computed by
taking a gradient of the criteria evaluated at the optimal points $Q^*$ or $f^*$.
%

Despite the theoretical equivalence of both approaches, practical considerations lead to different behaviours and to potentially poor approximations of the real gradients.
For example, in the dual formulation, one usually restricts the witness functions to be smooth, while in the primal formulation, the constraint on $Q$ is only approximately enforced. We will study the effect of these approximations.

\paragraph{Imperfect gradients in the dual (i.e., for WGAN)}
We show that (i) if the true optimum $f^*$ is not reached exactly (no matter how close), the effect on the gradient in the dual formulation can be arbitrarily large, and (ii) this also holds when the
optimization is performed only in a restricted class of smooth functions.
We write the criterion to be optimized as $J_D(f)\bydef \ee{P_X}{f(X)} - \ee{P_Z}{f(G(Z))}$ and
denote its gradient with respect to $G$ by $\nabla J_D(f)$. Let $\H$ be a subset of the $1$-Lipschitz functions $\F_L$ on $\X$ containing smooth functions with bounded Hessian. Denote by $f_\H^*$ the minimizer of $J_D$ in $\H$. $A(f,f')\bydef cos\left(\nabla J_D(f),\nabla J_D(f')\right)$ will
denote the cosine of the angle between the gradients of the criterion at different functions.
\begin{proposition}\label{prop:gradients}
There exists a constant $C>0$ such that for any $\epsilon > 0$, one can construct distributions $P_X$,
$P_G$ and pick witness functions $f_\epsilon \in \mathcal{F}_L$ and $h_{\epsilon}\in \H$ that are $\epsilon$-optimal $|J_D(f_\epsilon) -J_D(f^*)|\le \epsilon$, $|J_D(h_{\epsilon}) -J_D(h^*)|\le\epsilon$, but which give (at some point $z\in\Z$) gradients whose direction is at least $C$-wrong:
$A(f_\epsilon,f^*) \le  1-C$, $A(h_0,h^*) \le 1-C$, and $A(h_\epsilon, h^*) \le 0$.
\end{proposition} 

\paragraph{Imperfect posterior in the primal (i.e., for POT)}
In the primal formulation, when the constraint is violated, that is the aggregated posterior $Q_Z$ is not matching $P_Z$, there can be two kinds of negative effects: 
(i) the gradient of the criterion is only computed on a (possibly small) subset of the latent space reached by $Q_Z$;
(ii) several input points could be mapped by $Q(Z|X)$ to the same latent code $z$, thus giving gradients that encourage $G(z)$ to be the average/median of several inputs (hence encouraging a blurriness). See Section \ref{sec:VAEs} for the details and Figure \ref{fig:main} for an illustration.

\section{Conclusion}
\label{sec:conclusion}
This work proposes a way to fit generative models by minimizing any optimal transport cost.
It also establishes novel links between different popular unsupervised probabilistic modeling techniques.
Whilst our contribution is on the theoretical side, it is reassuring to note that the empirical results of~\cite{MSJ+16} show the strong performance of our method for the special case of the 2-Wasserstein distance. Experiments with other cost functions $c$ are beyond the scope of the present work and left for the future studies.

\section*{Acknowledgments}
The authors are thankful to Mateo Rojas-Carulla and Fei Sha for stimulating discussions.
CJSG is supported by a Google European Doctoral Fellowship in Causal Inference.

\bibliographystyle{unsrt}
\bibliography{nips2017vegan}

\newpage
\appendix 
\section{Further details on VAEs and GANs}
\label{appendix:VAE}
\paragraph{VAE, KL-divergence and a  marginal log-likelihood}
For models $P_G$ of the form \eqref{eq:latent-var} and \emph{any} conditional distribution $Q(Z|X)$ it can be easily verified that
\begin{align}
\notag
-\E_{P_X}[ \log P_G(X) ]
=
&-\E_{P_X}\bigl[
D_{\KL}\bigl( Q(Z|X), P_G(Z|X) \bigr)
\bigr]\\
\label{eq:app-elbo}
&+
\E_{P_X}\left[
D_\KL\bigl(Q(Z|X), P_Z\bigr)
-
\E_{Q(Z|X)}[\log p_G(X|Z)]
\,
\right].
\end{align}
Here the conditional distribution $P_G(Z|X)$ is induced by a joint distribution $P_{G,Z}(X,Z)$,
which is in turn specified by the 2-step latent variable procedure: (a) sample $Z$ from $P_Z$, (b) sample $X$ from $P_G(X|Z)$.
Note that the first term on the r.h.s.\:of \eqref{eq:app-elbo} is always non-positive, while
the l.h.s.\:does not depend on $Q$.
This shows that if conditional distributions $Q$ are not restricted then
\[
-\E_{P_X}[ \log P_G(X) ] = \inf_Q
\E_{P_X}\left[
D_\KL\bigl(Q(Z|X), P_Z\bigr)
-
\E_{Q(Z|X)}[\log p_G(X|Z)]
\,
\right],
\]
where the infimum is achieved for $Q(Z|X) = P_G(Z|X)$.
However, for any restricted class $\mathcal{Q}$ of conditional distributions $Q(Z|X)$ we only have
\begin{align*}
&-\E_{P_X}[ \log P_G(X) ]\\
&=
\inf_{Q}\!-\E_{P_X}\bigl[
D_{\KL}\bigl( Q(Z|X), P_G(Z|X) \bigr)
\bigr]
\!+\!
\E_{P_X}\left[
D_\KL\bigl(Q(Z|X), P_Z\bigr)
\!-\!
\E_{Q(Z|X)}[\log p_G(X|Z)]
\,
\right]\\
&\leq
\inf_{Q\in\mathcal{Q}}
\E_{P_X}\left[
D_\KL\bigl(Q(Z|X), P_Z\bigr)
-
\E_{Q(Z|X)}[\log p_G(X|Z)]
\,
\right] = D_\VAE(P_X,P_G),
\end{align*}
where the inequality accounts for the fact that $Q(Z|X)$ might be not flexible enough to match $P(Z|X)$ for all values of $X$.

\paragraph{Relation between AAE, AVB, and VAE}
\begin{proposition}
For any distributions $P_X$ and $P_G$:
\[
D_\AAE(P_X, P_G)
\leq
D_\AVB(P_X, P_G).
\]
\end{proposition}
\begin{proof}
By Jensen's inequality and the joint convexity of $D_\GAN$ we have
\begin{align*}
D_\AAE(P_X, P_G) &= 
\inf_{Q(Z|X)\in\mathcal{Q}} 
D_\GAN\left(
{\textstyle\int_{\X} Q(Z|x)p_X(x)dx, P_Z}
\right)
-
\E_{P_X}\E_{Q(Z|X)}[\log p_G(X|Z)]\\
&\leq
\inf_{Q(Z|X)\in\mathcal{Q}} 
\E_{P_X}\left[D_\GAN\left(
{Q(Z|X)p_X(x)dx, P_Z}
\right)
-
\E_{Q(Z|X)}[\log p_G(X|Z)]\right]\\
& = D_\AVB(P_X,P_G).
\end{align*}
\end{proof}

Under certain assumptions it is also possible to link $D_\AAE$ to $D_\VAE$:
\begin{proposition}
Assume $D_\KL(Q(Z|X), P_Z) \geq 1/4$ for all $Q\in\mathcal{Q}$ with $P_X$-probability 1. Then
\[
D_\AAE(P_X,P_G) \leq D_\VAE(P_X,P_G).
\]
\end{proposition}
\begin{proof}
We already mentioned that 
$
D_\GAN(P, Q)\leq2\cdot D_\JS(P, Q) - \log(4)
$
for any distributions $P$ and~$Q$.
Furthermore, $D_\JS(P,Q) \leq \frac{1}{2}D_{\mathrm{TV}}(P,Q)$ \cite[Theorem 3]{L91}
and $D_{\mathrm{TV}}(P,Q) \leq \sqrt{D_{\KL}(P,Q)}$ \cite[Eq.\,2.20]{T08}, which leads to
\[
D_\JS(P,Q) \leq \frac{1}{2} \sqrt{D_{\KL}(P,Q)}.
\]
Together with the joint convexity of $D_\JS$ and Jensen's inequality this implies
\begin{align*}
D_\AAE(P_X, P_G) &:= 
\inf_{Q(Z|X)\in\mathcal{Q}} 
D_\GAN\left(
{\textstyle\int_{\X} Q(Z|x)p_X(x)dx, P_Z}
\right)
-
\E_{P_X}\E_{Q(Z|X)}[\log p_G(X|Z)]\\
&\leq
\inf_{Q(Z|X)\in\mathcal{Q}} 
D_\JS\left(
{\textstyle\int_{\X} Q(Z|x)p_X(x)dx, P_Z}
\right)
-
\E_{P_X}\E_{Q(Z|X)}[\log p_G(X|Z)]\\
&\leq
\inf_{Q(Z|X)\in\mathcal{Q}} 
\E_{P_X}\left[
D_\JS\left(
{Q(Z|X), P_Z}
\right)
-
\E_{Q(Z|X)}[\log p_G(X|Z)]\right]\\
&\leq
\inf_{Q(Z|X)\in\mathcal{Q}} 
\E_{P_X}\left[
\frac{1}{2}\sqrt{D_\KL\left(
{Q(Z|X), P_Z}
\right)}
-
\E_{Q(Z|X)}[\log p_G(X|Z)]\right]\\
&\leq
\inf_{Q(Z|X)\in\mathcal{Q}} 
\E_{P_X}\left[
D_\KL\left(
{Q(Z|X), P_Z}
\right)
-
\E_{Q(Z|X)}[\log p_G(X|Z)]\right]\\& = D_\VAE(P_X,P_G).
\end{align*}
\end{proof}

\section{Proofs}
\label{appendix:proofs}

\subsection{Proof of Theorem \ref{thm:main}}
We start by introducing an important lemma relating the two sets over which $W_c$ and $W_c^\dagger$ are optimized.
\begin{lemma}\label{lemma:main}
$\mathcal{P}_{X,Y} \subseteq \mathcal{P}(P_X,P_G)$ with identity if $P_G(Y|Z=z)$ are Dirac distributions for all $z\in\Z$\footnote{We conjecture that this is also a necessary condition. The necessity is not used in the remainder of the paper.}.
\end{lemma}
\begin{proof}
The first assertion is obvious.
To prove the identity, note that when $Y$ is a deterministic function of~$Z$, for any $A$ in the sigma-algebra induced by $Y$ we have $\e{\oo{Y\in A}|X,Z}=\e{\oo{Y\in A}|Z}$. This implies $(Y\independent X)|Z$ and concludes the proof.
\end{proof}

\label{appendix:thm:main}
Inequality \eqref{eq:thm-main-1} and the first identity in \eqref{eq:thm-main-2} obviously follows from Lemma \ref{lemma:main}.
The tower rule of expectation, and the conditional independence property of $\mathcal{P}_{X,Y,Z}$ implies
\begin{align*}
W_c^\dagger(P_X, P_G)
&=
\inf_{P\in \mathcal{P}_{X,Y,Z}} \ee{(X, Y,Z)\sim P}{c(X,Y)}\\
&=
\inf_{P\in \mathcal{P}_{X,Y,Z}} \E_{P_Z} \E_{X\sim P(X|Z)} \E_{Y\sim P(Y|Z)} [c(X,Y)]\\
&=
\inf_{P\in \mathcal{P}_{X,Y,Z}} \E_{P_Z} \E_{X\sim P(X|Z)} \bigl[c\bigl(X,G(Z)\bigr)\bigr]\\
&=
\inf_{P\in \mathcal{P}_{X,Z}} \E_{(X,Z)\sim P} \bigl[c\bigl(X,G(Z)\bigr)\bigr].
\end{align*}

\subsection{Random decoders $P_G(Y|Z)$}
\label{appendix:corr-gauss}

\begin{corollary}
\label{corr:gauss}
Let $\X=\R^d$ and assume the conditional distributions $P_G(Y|Z=z)$ have mean values $G(z)\in\R^d$ and marginal variances $\sigma_1^2,\dots,\sigma_d^2\geq0$ for all $z\in\Z$, where $G\colon \Z \to \X$. 
Take $c(x,y) = \|x - y\|^2_2$. Then
\begin{equation}
\label{eq:upper-bound-gauss}
W_c(P_X, P_G)
\leq
W_c^\dagger(P_X, P_G)
=
\sum_{i=1}^d \sigma^2_i
+
\inf_{P\in \mathcal{P}(X\sim P_X, Z\sim P_Z)} \E_{(X,Z)\sim P}\bigl[\| X - G(Z)\|^2\bigr].
\end{equation}
\end{corollary}
\begin{proof}
Proof is similar to the one of Theorem \ref{thm:main}. See Section \ref{appendix:corr-gauss}.
\end{proof}

First inequality follows from \eqref{eq:thm-main-1}.
For the identity we proceed similarly to the proof of Theorem \ref{thm:main} and write
\begin{equation}
\label{eq:proof-gauss}
W_c^\dagger(P_X, P_G)
=
\inf_{P\in \mathcal{P}_{X,Y,Z}} \E_{P_Z} \E_{X\sim P(X|Z)} \E_{Y\sim P(Y|Z)} \bigl[\|X - Y\|^2\bigr].
\end{equation}
Note that
\begin{align*}
&\E_{Y\sim P(Y|Z)} \bigl[\|X - Y\|^2\bigr]
=
\E_{Y\sim P(Y|Z)} \bigl[\|X - G(Z) + G(Z) - Y\|^2\bigr]\\
&=
\|X - G(Z)\|^2
+
\E_{Y\sim P(Y|Z)} \bigl[\langle
X - G(Z), G(Z) - Y
\rangle\bigr]
+
\E_{Y\sim P(Y|Z)}\|G(Z) - Y\|^2\\
&=
\|X - G(Z)\|^2 + \sum_{i=1}^d \sigma_i^2.
\end{align*}
Together with \eqref{eq:proof-gauss} and the fact that $\mathcal{P}_{X,Z}=\mathcal{P}(X\sim P_X, Z\sim P_Z)$ this concludes the proof.

\subsection{Proof of Proposition \ref{prop:vaes}}
\label{appendix:prop-vaes}

Corollary \ref{corr:gauss} shows that $G^\dagger$ does not depend on the variance $\sigma^2$.
When $\sigma^2=0$ the distribution $P_G(Y|Z)$ turns into Dirac.
In this case we combine Theorem \ref{thm:main} and Corollary \ref{corr:gauss} to conclude that $G^\dagger$ also minimizes $W_c(P_X,P_G^0)$.
Next we prove that $G^*_{\sigma}$ generally depends on~$\sigma^2$.

We will need the following simple result, which is basically saying that the variance of a sum of two independent random variables is a sum of the variances:
\begin{lemma}
\label{lemma:app-prop}
Under conditions of Proposition \ref{prop:vaes}, assume $Y \sim P_G^{\sigma}$. Then
\[
\mathrm{Var}[ Y ] = \sigma^2 + \mathrm{Var}_{Z\sim P_Z}[G(Z)].
\]
\end{lemma}
\begin{proof}
First of all, using \eqref{eq:latent-var} we have
\[
\E[Y] := 
\int_\R y \int_\Z p_G(y|z) p_Z(z) dz dy= 
\int_\Z \left(\int_\R y  \;p_G(y|z) dy\right) p_Z(z) dz
=
\E_{Z\sim P_Z}[G(Z)].
\]
Then
\begin{align*}
\mathrm{Var}[ Y ] &:=
\int_\R (y - \E[G(Z)])^2 \int_\Z p_G(y|z) p_Z(z) dz dy\\
&=
\int_\R (y - G(z))^2 \int_\Z p_G(y|z) p_Z(z) dz dy
+
\int_\R (G(z) - \E[G(Z)])^2 \int_\Z p_G(y|z) p_Z(z) dz dy\\
&=
\sigma^2 + \mathrm{Var}_{Z\sim P_Z}[G(Z)].
\end{align*}
\end{proof}

Next we prove the remaining implication of Proposition \ref{prop:vaes}. 
Namely, that when $\sigma^2 > 0$ the function $G^*_{\sigma}$ minimizing $W_c(P_X, P_G^\sigma)$ depends on $\sigma^2$.
The~proof is based on the following example: $\X = \Z = \R$, $P_X = \mathcal{N}(0,1)$, $P_Z = \mathcal{N}(0,1)$, and $0<\sigma^2<1$.
Note that by setting $G(z) = c\cdot z$ for any $c > 0$ we ensure that $P_G^{\sigma}$ is the Gaussian distribution, because a convolution of two Gaussians is also Gaussian.
In particular if we take $G^*(z) = \sqrt{1-\sigma^2}\cdot z$ Lemma \ref{lemma:app-prop} implies that $P_{G^*}^\sigma$ is the standard normal Gaussian $\mathcal{N}(0,1)$.
In other words, we obtain the global minimum $W_c(P_X, P_{G^*}^\sigma) = 0$ and $G^*$ clearly depends on $\sigma^2$.

\subsection{Proof of Proposition \ref{prop:gradients}}
\label{appendix:prop-grad}
We just give a sketch of the proof: consider discrete distributions $P_X$ supported on two points 
$\{x_0,x_1\}$, and $P_Z$ supported on $\{0,1\}$ and let $y_0=G(0)$, $y_1=G(1)$ ($y_0\ne y_1$). Given an optimal $f^*$, one can modify locally it around $y_0$ without changing its Lipschitz
constant such that the obtained $f_\epsilon$ is an $\epsilon$-approximation of $f^*$ whose gradients
at $y_0$ and $y_1$ point in directions arbitrarily different from those of $f^*$.
For smooth functions, by moving $y_0$ and $y_1$ away from the segment $[x_0,x_1]$ but close to
each other $\|y_0-y_1\|\le K\epsilon$, the gradients of $f^*$ will point in directions roughly opposite but
the constraint on the Hessian will force the gradients of $f_{\F,0}$ at $y_0$ and $y_1$ to be very close.
Finally, putting $y_0,y_1$ on the segment $[x_0,x_1]$, one can get an $f_\F^*$ whose gradients at
$y_0$ and $y_1$ are exactly opposite, while taking $f_{\F,\epsilon}(y)=f_\F^*(y+\epsilon)$, we can
swap the direction at one of the points while changing the criterion by less than $\epsilon$.

\end{document}